\documentclass[11pt,a4paper]{article}
\usepackage[utf8]{inputenc}
\usepackage{authblk}
\usepackage[top=25truemm,bottom=25truemm,left=25truemm,right=25truemm,footskip=10truemm]{geometry}
\usepackage{booktabs}
\usepackage{amsmath}
\usepackage{amsthm}
\usepackage{float}
\usepackage{amssymb}
\usepackage{amsfonts}
\usepackage{ascmac}
\usepackage{scalefnt}
\usepackage{here}
\usepackage{indentfirst}
\usepackage{algorithmic}
\usepackage{algorithm}
\usepackage{setspace}
\usepackage{color}
\usepackage[pdftex]{graphicx}
\usepackage{subcaption}
\usepackage{lscape}
\usepackage{natbib}
\usepackage{url}
\usepackage{placeins}
\usepackage{tikz}
\DeclareMathOperator*{\argmin}{arg\,min}

\newtheorem{theorem}{Theorem}[section]
\newtheorem{proposition}[theorem]{Proposition}

\newtheorem{lemma}[theorem]{Lemma}
\newtheorem{assumption}{Assumption}[section]

\newtheorem{remark}{Remark}[section]

\title{Median Consensus Embedding for Dimensionality Reduction}
\author[1,*]{Yui Tomo}
\author[1]{Daisuke Yoneoka}
\affil[1]{Department of Epidemiology, National Institute of Infectious Diseases, Japan Institute for Health Security, 1-23-1 Toyama, Shinjuku-Ku, Tokyo 162-0052, Japan}
\affil[*]{E-mail: tomo.y@jihs.go.jp}
\date{}

\begin{document}
\maketitle

\begin{abstract}
This study proposes median consensus embedding (MCE) to address variability in low-dimensional embeddings caused by random initialization in nonlinear dimensionality reduction techniques such as $t$-distributed stochastic neighbor embedding.
MCE is defined as the geometric median of multiple embeddings.
By assuming multiple embeddings as independent and identically distributed random samples and applying large deviation theory, we prove that MCE achieves consistency at an exponential rate.
Furthermore, we develop a practical algorithm to implement MCE by constructing a distance function between embeddings based on the Frobenius norm of the pairwise distance matrix of data points.
Application to actual data demonstrates that MCE converges rapidly and effectively reduces instability.
We further combine MCE with multiple imputation to address missing values and consider multiscale hyperparameters.
Results confirm that MCE effectively mitigates instability issues in embedding methods arising from random initialization and other sources.
\end{abstract}


\section{Introduction}

Dimensionality reduction is a fundamental approach for exploring and visualizing high-dimensional data. 
Linear dimensionality reduction methods, such as principal component analysis (PCA), project the data onto low-dimensional linear subspaces~\citep{pearson1901liii}.
However, these approaches assume a global linear structure and often fail to adequately capture nonlinear structures such as complex clusters and trajectories. 
To address these limitations, nonlinear methods have been developed, such as $t$-distributed stochastic neighbor embedding ($t$-SNE) and uniform manifold approximation and projection (UMAP).
These methods aim to preserve local neighborhood relationships or manifold geometry in a low-dimensional embedding, and in many fields, they have become standard tools for visualizing high-dimensional data~\citep{van2008visualizing, mcinnes2018umap}. 
However, these nonlinear methods are sensitive to local optima, resulting in variability in embeddings even under the same parameter settings. This instability arises from random initialization.
Such instability may lead to misleading or inconsistent interpretations of the underlying data structure.

In practice, it is common to rerun these algorithms multiple times and to select, often implicitly, those configurations that align closely with the prior hypothesis.
However, such a post-selection procedure may lower reproducibility.
One simple approach to reduce such variability is to fix the initialization in a deterministic manner, for example, by using embeddings obtained from linear dimensionality reduction methods as initial values~\citep{kobak2019art}.
Such initializations may reduce variation, but they impose a particular linear structure in initialization and do not eliminate sensitivity to the local optima of the methods.
These considerations motivate the development of statistical procedures for integrating multiple embeddings into a single consensus representation.

Integrative approaches are commonly employed in various statistical and machine learning fields. 
For example, in the development of prediction models, ensemble-learning methods, such as bagging and boosting, are commonly used to achieve robust and stable results~\citep{breiman1996bagging,freund1995desicion,hastie2009ensemble}. 
Similarly, in cluster analysis, consensus clustering or cluster ensemble approaches have been proposed to combine the outputs of clustering algorithms prone to initialization sensitivity, such as $k$-means, to produce more robust and stable clustering results~\citep{strehl2002cluster, topchy2004analysis}.

In a similar spirit, the consensus-locally linear embedding (C-LLE) method constructs a consensus representation by combining multiple embeddings obtained by the locally linear embedding method~\citep{tiwari2008consensus}. 
Related works include a consensus embedding framework proposed by~\cite{viswanath2012consensus}.
They first evaluate each embedding according to how well it preserves the relative distances among triplets of points when compared with the original high-dimensional representation, then retain those embeddings that achieve high scores and aggregate the corresponding low-dimensional dissimilarities to form a consensus dissimilarity matrix.
These frameworks integrate multiple embeddings by aggregating pairwise dissimilarities between embeddings and applying projection methods such as multidimensional scaling.
However, these methods operate at the level of pairwise dissimilarities of embedded points and do not endow the space of embeddings with a geometric or probabilistic structure, and the input embeddings are treated as fixed rather than as random outputs of stochastic algorithms.
Consequently, these approaches do not provide statistical guarantees on how the consensus embedding reduces variability arising from random initialization.

In this study, we formalize the consensus embedding problem based on a metric function on an embedding space and propose a new consensus embedding approach.
Specifically, we make the following contributions.
\begin{enumerate}
    \item We propose median consensus embedding (MCE) as a novel approach for integrating multiple embeddings.
    We define an embedding space as a quotient space whose elements are equivalence classes with fixed location and scale, and identified up to rotations and reflections.
    We further define a proper metric on it and define MCE as the geometric median of multiple embeddings.
    Specifically, MCE is the embedding that minimizes the average distance between them.
    \item We prove that MCE converges to a population target embedding at an exponential rate as the number of embeddings increases by modeling each embedding as an independent and identically distributed (i.i.d.) random element drawn from a probability measure on the embedding space.
    We apply tools from large-deviation theory to prove this property~\citep{dembo2009large}.
    \item We develop a practical algorithm for implementing MCE by constructing a concrete distance function on the embedding space based on a distance function of the pairwise distance matrix of the data points.
    We prove that the optimization problem regarding embeddings is equivalent to that of pairwise distance matrices.
    \item We conduct empirical evaluations to demonstrate rapid convergence and reduction in embedding instability as the number of randomly initialized embeddings increases.
    We further illustrate a combined approach with multiple imputation and investigate the integration of embeddings obtained under multiscale hyperparameters.
\end{enumerate}

The remainder of this paper is organized as follows: 
In Section~\ref{sec:preliminaries}, we introduce the problem setup and related works. In Section~\ref{sec:method}, we describe our proposed method and provide a theoretical analysis. In Section~\ref{sec:algorithm}, we present the implementation algorithm developed for the proposed method. In Section~\ref{sec:illustration}, we provide illustrations and empirical evaluations of the proposed method using actual data. In Section~\ref{sec:discussion}, we discuss the strengths of our method and future research.

\section{Preliminaries and Related Works}
\label{sec:preliminaries}

We first describe the problem setup in this study.
We then introduce nonlinear dimensionality reduction methods and some existing consensus approaches as related works.

\subsection{Problem Setup}

Let $q$ and $p$ denote the input and embedding dimensions, respectively ($q > p$), and let $h \in \mathbb{R}^{n\times q}$ denote the input high-dimensional data matrix.
Dimensionality reduction methods with target dimension $p$ map $h$ to $p$-dimensional representations.
Let $y\in \mathbb{R}^{n \times p}$ denote a set of data points in the $p$-dimensional space.
Suppose that we have $m$ embeddings $y_1,y_2,\ldots,y_m \in \mathbb{R}^{n \times p}$ obtained by dimensionality reduction methods.
This study aims to develop procedures to integrate these multiple embeddings into a single consensus representation and provide a statistical guarantee regarding stability.

\subsection{Nonlinear Dimensionality Reduction Methods}

Nonlinear dimensionality reduction methods are designed to preserve some structures of the data.
Isomap approximates the geodesic distances via a neighborhood graph and applies multidimensional scaling (MDS) to obtain a low-dimensional embedding~\citep{tenenbaum2000global}.
Locally linear embedding (LLE) preserves local linear reconstruction weights within neighborhoods~\citep{belkin2003laplacian}.
Laplacian eigenmaps preserve the local proximity with respect to a graph Laplacian constructed from nearest-neighbor relations~\citep{roweis2000nonlinear}.
Diffusion maps use a diffusion process on the data graph and embed points to preserve diffusion distances~\citep{coifman2005geometric}.
Potential of heat-diffusion for affinity-based trajectory embedding (PHATE) uses the $t$-step potential distance via a diffusion process that can capture the data's continuous trajectory structure, especially for branching data.~\citep{moon2019visualizing}.
$t$-distributed stochastic neighbor embedding ($t$-SNE) preserves the local structure of data by minimizing the Kullback-Leibler divergence loss between neighboring probabilities in the high-dimensional space based on Gaussian distributions and the corresponding neighboring probabilities in the low-dimensional space based on $t$-distributions~\citep{van2008visualizing}.
Uniform manifold approximation and projection (UMAP) also preserves neighborhood probabilities but avoids computing normalization factors and uses cross-entropy loss for fast computation~\citep{mcinnes2018umap}.

As a representative and widely used method, we specifically describe $t$-SNE.
Let the $i$th row of $h$ be denoted by $h_{(i)}$.
Define similarity between pairs of data points $\{h_{(i)},h_{(j)}\}$ for $i=1,\ldots,n$ and $j=1,\ldots,n$ as:
\begin{align*}
    p_{i j}(h)
    &:=
    \frac{
        p_{i \mid j}(h)
        +
        p_{j \mid i}(h)
    }{2 n}
    ,
    \quad \text{where} \quad
    p_{i \mid j}(h)
    :=
    \frac{\exp \left(-\|h_{(i)}-h_{(j)}\|_2^2 / 2 \sigma_j^2\right)}{\sum_{k \neq j} \exp \left(-\|h_{(k)}-h_{(j)}\|_2^2 / 2 \sigma_j^2\right)}
    ,
\end{align*}
with the given scale parameters $\sigma_1,\ldots,\sigma_n > 0$, which are determined by a $\mathrm{perplexity}$ parameter.
Specifically, the values of scale parameters are determined such that
\begin{align*}
    2^{-\sum_j p_{i \mid j}(h) \log_2 p_{i \mid j}(h)} = \mathrm{perplexity}
    .
\end{align*}
Then, define similarity between pairs of data points $\{y_{(i)},y_{(j)}\}$ as
\begin{align*}
    q_{i j}(y)
    &:=
    \frac{\left(1+\|h_{(i)}-h_{(j)}\|_2^2\right)^{-1}}{\sum_{k \neq t}\left(1+\|h_{(k)}-h_{(t)}\|_2^2\right)^{-1}}
    .
\end{align*}
The data points embedded into a $p$-dimensional space are obtained as a minimizer of the Kullback–Leibler divergence loss function:
\begin{align*}
    \argmin_{y \in \mathbb{R}^{p}}
    \left\{
    \sum_{i=1}^{n}\sum_{j=1}^{n} p_{i j}(h) \log \frac{p_{i j}(h)}{q_{i j}(y)}
    \right\}
    .
\end{align*}
Generally, the gradient descent algorithm is employed to compute a resulting embedding.
Because $t$-SNE and UMAP provide clear visual separation of cluster structures, many fields have recently adopted them as standard tools for visualizing high-dimensional data.
However, their solutions are sensitive to initialization due to the nonconvexity of the loss function.

\subsection{Consensus Embedding Approaches}

We describe two existing consensus embedding approaches.
The consensus-locally linear embedding (C-LLE) was proposed by~\cite{tiwari2008consensus}.
In C-LLE, multiple embeddings $y_1,\ldots,y_m \in \mathbb{R}^{n\times p}$ are obtained by applying LLE to the same data matrix $h$ with different neighborhood sizes.
For each embedding $y_b$ and each pair of data points $\{y_{b(i)}, y_{b(j)}\}$, a low-dimensional distance is defined as $W_{ij}(y_b) := \| y_{b(i)} - y_{b(j)}\|_2$, for $i,j=1,\ldots,n$.
C-LLE regards the collection $\{W_{ij}(y_b)\}_{b=1}^m$ as noisy observations of an unknown consensus dissimilarity $\tilde{W}_{ij}$ and estimates $\tilde{W}_{ij}$ by the maximum likelihood estimator (MLE) under a specified probabilistic model.
A final low-dimensional embedding is then obtained by applying MDS to the estimated consensus dissimilarity matrix $(\tilde{W}_{ij})_{1\leq i,j\leq m}$.
Thus, C-LLE aggregates pairwise distances from multiple LLE embeddings via MLE to construct a consensus embedding.
\cite{viswanath2012consensus} also proposed a consensus embedding framework.
Let $\Lambda_{ij}$ denote a dissimilarity between input points $h_{(i)}$ and $h_{(j)}$.
For a given embedding $y$, let $\delta_{ij}(y)$ denote the corresponding dissimilarity between the embedded points $y_{(i)}$ and $y_{(j)}$.
Based on these dissimilarities, they define a triangle-relationship indicator
\begin{align*}
    \Delta(i,j,k,y)
    :=
\begin{cases}
    1, & \text{if}\, \Lambda_{ij} < \Lambda_{ik} \,\text{and}\, \Lambda_{ij} < \Lambda_{jk}, \,\text{then}\, \delta_{ij}(y) < \delta_{ik}(y) \,\text{and}\, \delta_{ij}(y) < \delta_{jk}(y)
    ,
    \\ 
    0, & \text{otherwise}
    .
\end{cases}
\end{align*}
They further define the corresponding embedding strength as
\begin{align*}
    \psi(y)
    :=
    \frac{1}{|\mathcal{T}|}
    \sum_{(i,j,k)\in\mathcal{T}} \Delta(i,j,k,y)
    ,
\end{align*}
where $\mathcal{T}:= \{(i,j,k): 1 \le i < j < k \le n\}$ is the set of all unordered triplets of different indices.
$\psi(y)$ implies the proportion of triplets for which the relative pairwise relationships are preserved.
An embedding with $\psi(y)=1$ is regarded as a true embedding, and embeddings with $\psi(y) > \theta$ are termed strong, where $\theta \in [0,1]$ is a pre-specified threshold.
Given $k$ strong embeddings $y^{(1)},\ldots,y^{(k)} \in \{y_1,\ldots,y_m\}$, they define a consensus dissimilarity for each pair $(i,j)$ by aggregating the corresponding dissimilarities across these embeddings: $\tilde{W}_{ij} = \Omega\left(\left\{W_{ij}(y^{(1)}),\dots,W_{ij}(y^{(k)})\right\}\right)$, where $\Omega: \mathbb{R}^{k} \to \mathbb{R}$ denotes an aggregation rule, such as mean or median.
The resulting dissimilarity matrix $(\tilde{W}_{ij})_{1\leq i,j\leq m}$ is then used as input to a projection method, such as MDS, to obtain a consensus embedding.

Both C-LLE and the method proposed by~\cite{viswanath2012consensus} operate at the level of pairwise dissimilarities.
In other words, they first construct or select a collection of embeddings, aggregate the corresponding low-dimensional dissimilarities to obtain a consensus dissimilarity matrix, and then apply a projection method to produce a single embedding. 
Therefore, these methods do not endow the embedding space with any metric structure.
Furthermore, in these approaches, the embeddings $y_1,\ldots,y_m$ are treated as deterministic objects, and thus the randomness of inputs is not modeled explicitly.
Their analyses, therefore, do not provide any statistical guarantees on how the consensus embedding reduces the variability arising from a random initialization.
The current study addresses this problem by defining a metric on a defined embedding space and proposes a geometric definition of a consensus embedding.

\section{Proposed Method}
\label{sec:method}

In this section, we present our integration approach that uses the geometric median of embeddings.
We first define a metric space of embeddings and then formulate the proposed method.
We further prove that the method achieves consistency with an exponential rate when embeddings are modeled as random elements with a defined probability measure.

\subsection{Definition of Embedding Space and Metric}

Low-dimensional embeddings are typically defined only up to translation, scaling, and orthogonal transformations (rotations and reflections).
Therefore, we define the embedding space as a quotient space whose elements are equivalence classes with fixed location and scale, and identified up to rotations and reflections.
First, we define
\begin{align*}
    \mathcal{Y}
    :=
    \left\{y \in \mathbb{R}^{n \times p}
    \mid
    \cfrac{1}{n} \sum_{i=1}^n y_{(i)} = \mathbf{0},\,
    \cfrac{1}{n} \sum_{i=1}^n\left\|y_{(i)}\right\|_2^2=1 \right\}
    ,
\end{align*}
where $y_{(i)}$ is the $i$th row of $y$ and $\mathbf{0}$ denotes the $p$-dimensional zero vector.
We then regard the members of $\mathcal{Y}$ as equivalent up to rotations and reflections, and we introduce the following equivalence relation $\sim$:
\begin{align*}
    y_1 \sim y_2
    \quad \Longleftrightarrow \quad
    \exists R \in \Omega(p)
    \quad \text{such that} \quad
    y_2^{\top} = R\,y_1^{\top},
\end{align*}
where $\Omega(p) \subset \mathbb{R}^{p \times p}$ denotes the set of all $p \times p$ orthogonal matrices.
We introduce the set of equivalence classes
\begin{align*}
    \tilde{\mathcal{Y}}
    :=
    \mathcal{Y}/\sim
    =
    \{ [y] \mid y \in \mathcal{Y} \},
\end{align*}
where, for each $y \in \mathcal{Y}$,
\begin{align*}
    [y]
    :=
    \{ z \in \mathcal{Y} \mid \exists R \in \Omega(p)
    \quad \text{such that} \quad
    z^{\top} = R y^{\top} \}.
\end{align*}
We then define the quotient map
\begin{align*}
    \pi : \mathcal{Y} \to \tilde{\mathcal{Y}},
    \quad
    \pi(y) = [y].
\end{align*}
We hereafter consider embeddings as members of $\tilde{\mathcal{Y}}$.

\begin{remark}
Since $\mathcal{Y}$ is a closed and bounded subset of $\mathbb{R}^{n\times p}$, $\mathcal{Y}$ is compact. 
Moreover, the set of orthogonal matrices $\Omega(p)$ is compact.
For each $y \in \mathcal{Y}$, the equivalence class $\pi(y)$ can be written as
\begin{align*}
    \pi(y)
    =
    \{ (R y^{\top})^{\top} : R \in \Omega(p) \}
    ,
\end{align*}
which is the image of the compact set $\Omega(p)$ under the continuous map
\begin{align*}
    \Phi : \Omega(p) \to \mathcal{Y}
    ,
    \quad
    \Phi(R) = (R y^{\top})^{\top}
    .
\end{align*}
Therefore, each equivalence class $\pi(y)$ is compact.
Furthermore, the quotient map $\pi : \mathcal{Y} \to \tilde{\mathcal{Y}}$ is continuous and surjective, and therefore its image $\tilde{\mathcal{Y}} = \pi(\mathcal{Y}) = \mathcal{Y}/\sim$ is also compact.
\end{remark}

Subsequently, we define the function 
$
    d : \tilde{\mathcal{Y}} \times \tilde{\mathcal{Y}} \to \mathbb{R}_{\geq 0}
$
which satisfies the following properties:
\begin{itemize}
    \item[(i)] 
    for all $y_1, y_2 \in \tilde{\mathcal{Y}}$,
    $
        d(y_1,y_2)=0 \, \Longleftrightarrow \, y_1 = y_2,
    $
    \item[(ii)] 
    for all $y_1, y_2 \in \tilde{\mathcal{Y}}$,
    $
        d(y_1,y_2) = d(y_2,y_1),
    $
    \item[(iii)]
    for all $y_1, y_2, y_3 \in \tilde{\mathcal{Y}}$,
    $
        d(y_1,y_2) + d(y_2,y_3) \geq d(y_1,y_3).
    $
\end{itemize}
Thus, $d(\cdot,\cdot)$ is a distance function on $\tilde{\mathcal{Y}}$.

\subsection{Median Consensus Embedding}

Suppose we have $m$ embeddings
$
    y_1,y_2,\dots,y_m \in \tilde{\mathcal{Y}}.
$
We define the MCE, denoted by $\hat{y}$, as the minimizer
\begin{align}
    \label{eq:consensus_embedding}
    \hat{y} := \argmin_{y \in \tilde{\mathcal{Y}}} \frac{1}{m} \sum_{i=1}^{m} d(y_i,y)
    .
\end{align}
In this formulation, $\hat{y}$ is the geometric median of $\{y_1,y_2,\dots,y_m\}$ in the metric space $(\tilde{\mathcal{Y}}, d)$.
Since the objective function in~\eqref{eq:consensus_embedding} is a continuous function on the compact metric space $(\tilde{\mathcal{Y}}, d)$, a minimizer always exists. 
Moreover, the MCE is invariant under translations, rescalings, and orthogonal transformations of the original low-dimensional embeddings obtained from dimensionality reduction methods, because it is defined on the quotient space $\tilde{\mathcal{Y}}$.

\subsection{Theoretical Analysis}

We then investigate the theoretical properties of the MCE in a probabilistic setting where embeddings are generated as random elements.
Let $\mu$ be a probability measure on $\tilde{\mathcal{Y}}$, and suppose that the embeddings $y_1,y_2,\dots,y_m$ are i.i.d. random elements with $\mu$.
We define the (population) true embedding $y^{*}$ as the solution to the optimization problem:
\begin{align*}
    y^{*}: = \argmin_{y\in\tilde{\mathcal{Y}}} \int_{\tilde{\mathcal{Y}}} d(y',y) d\mu(y')
    .
\end{align*}

For preparation, we introduce the following assumptions:

\begin{assumption}[Uniqueness of the true embedding]
\label{assump:trueembedding}
For any $y \in \tilde{\mathcal{Y}} \setminus \{y^{*}\}$, we assume that
\begin{align*}
    \int_{\tilde{\mathcal{Y}}} d(y',y) d\mu(y')
    >
    \int_{\tilde{\mathcal{Y}}} d(y',y^{*}) d\mu(y')
    .
\end{align*}
\end{assumption}

\begin{assumption}[Existence of moment generating function]
\label{assump:mgf}
Define 
$
    Z_{y}(y_i):= d(y_i,y^{*}) - d(y_i,y)
    .
$
Let $M_{y}(\lambda)$ denote the moment-generating function
\begin{align*}
    M_{y}(\lambda) := \mathbb{E}_{\mu}\left[\exp(\lambda Z_{y}(y_i))\right]
    .
\end{align*}
We assume that for all $y \in \tilde{\mathcal{Y}}$ and $\lambda \in \mathbb{R}$, we obtain
\begin{align*}
    \left|M_{y}(\lambda)\right| < \infty
    .
\end{align*}
\end{assumption}

Under these assumptions, we establish the following theorem on the consistency of MCE.

\begin{theorem}[Consistency with exponential rate]
\label{thm:main}
Suppose that Assumptions~\ref{assump:trueembedding} and~\ref{assump:mgf} are satisfied,
then for any $\epsilon > 0$, there exist $M \in \mathbb{N}$, $K>0$, and $\eta > 0$ such that if $m > M$, then
\begin{align*}
    \mathrm{Pr}\left( d(\hat{y},y^{*}) \geq \epsilon \right) \leq K \exp(- m \eta)
    .
\end{align*}
\end{theorem}

Theorem~\ref{thm:main} implies that the probability of deviation reduces at an exponential rate in the number of embeddings $m$.
Therefore, when Assumptions~\ref{assump:trueembedding} and~\ref{assump:mgf} are satisfied, there is no need to prepare an excessively large number of embeddings as inputs.
Even a moderate value of $m$ may be enough to ensure that $\hat{y}$ is close to $y^{*}$ with high probability. 
This provides a practical justification for using the MCE as a consensus representation when obtaining a large number of embeddings is computationally expensive.

\section{Algorithm Construction}
\label{sec:algorithm}

We develop a practical algorithm for applying the MCE to real-world data.
Since the embeddings are defined only up to translation, scaling, and orthogonal transformations (rotations and reflections), we first represent each embedding by the matrix of pairwise distances between its embedded points.
We then reformulate the optimization problem in terms of the distances between these distance matrices and construct an implementable algorithm for computing the MCE.

\subsection{Reformulation of the Optimization Problem}

We define the mapping $X: \mathcal{Y} \to \mathcal{X} \subset \mathbb{R}^{n \times n}$ as
\begin{align}
    \label{eq:def_X}
    X(y)_{(ij)}
    :=
    \| y_{(i)} - y_{(j)} \|_2,
    \quad \text{for } i,j=1,\ldots,n,
\end{align}
where $X(y)_{(ij)}$ denotes the $(i,j)$th entry of $X(y)$.
Furthermore, we define the mapping $\tilde{X}: \tilde{\mathcal{Y}} \to \mathcal{X}$ as
\begin{align*}
    \tilde{X}(y) := X(y'),
\end{align*}
where $y' \in \mathcal{Y}$ is a representative of $y \in \tilde{\mathcal{Y}}$.

\begin{remark}
Since $\mathcal{Y}$ is compact and $X$ is continuous, its image $\mathcal{X}$ is also compact.
\end{remark}

Let $D: \mathcal{X}\times\mathcal{X} \to \mathbb{R}_{\geq 0}$ be a distance function on $\mathcal{X}$.
We then obtain the following proposition regarding the construction of a distance function on $\tilde{\mathcal{Y}}$.

\begin{proposition}[Construction of distance function on $\tilde{\mathcal{Y}}$]
\label{prop:distance_x}
Define $d : \tilde{\mathcal{Y}} \times \tilde{\mathcal{Y}} \to \mathbb{R}_{\geq 0}$ by
\begin{align*}
    d(y_1,y_2)
    :=
    D\left(\tilde{X}(y_1),\tilde{X}(y_2)\right)
    .
\end{align*}
Then $d(\cdot,\cdot)$ is a distance function on $\tilde{\mathcal{Y}}$.
\end{proposition}

Based on this proposition, we define the following optimization problem:
\begin{align}
    \label{eq:opt_X}
    \hat{x} := \argmin_{x\in\mathcal{X}} \frac{1}{m}\sum_{i=1}^{m} D(x_i,x),
\end{align}
where $x_i := \tilde{X}(y_i)$ for $y_1,\ldots,y_m \in \tilde{\mathcal{Y}}$.
We now establish the following proposition regarding the equivalence of the optimization problems.

\begin{proposition}[Equivalence of optimization problems]
\label{prop:rewrite_opt}
Let $\hat{y}$ be a solution to the optimization problem~\eqref{eq:consensus_embedding} and $\hat{x}$ be a solution to~\eqref{eq:opt_X}.
Then $\tilde{X}(\hat{y}) = \hat{x}$ holds.
\end{proposition}

Therefore, by solving the optimization problem~\eqref{eq:opt_X} with respect to $x$, obtaining the optimal distance matrix $\hat{x}$, and then embedding the data points into the $p$-dimensional Euclidean space based on $\hat{x}$, we obtain an algorithm for computing the MCE.

\subsection{Implementation Algorithm}

The optimization problem (\ref{eq:opt_X}) reformulated using Proposition~\ref{prop:rewrite_opt} corresponds to the computation of the geometric median in the feature space $\mathcal{X}$. This can be numerically solved using the Weiszfeld algorithm. 

As a concrete distance function $D: \mathcal{X}\times\mathcal{X} \to \mathbb{R}_{\geq 0}$, we use the Frobenius norm
\begin{align*}
    D(x_1,x_2)
    =
    \|x_1 - x_2\|_{F}, \quad \text{for }x_1,x_2 \in \mathcal{X} \subset \mathbb{R}^{n \times n}
    .
\end{align*}
The solution $\hat{x}$ can be visualized in $\mathbb{R}^2$ via MDS.
Based on these components, we construct a practical MCE algorithm as Algorithm~\ref{alg:median_embedding}.

The computational cost of the initial computation of the distance matrices for each embedding is $\mathcal{O}(m\, n^2)$.
In each iteration of the Weiszfeld algorithm, computing the Frobenius norm for each distance matrix requires $\mathcal{O}(m\, n^2)$ operations. Therefore, if $T$ iterations are required, the computational cost of this step is $\mathcal{O}(T\,m\, n^2)$.
Finally, the MDS step involves the eigen decomposition of an $n\times n$ matrix, which typically has a cost of $\mathcal{O}(n^3)$. 
Therefore, the total computational cost is $\mathcal{O}((T+1)\,m\, n^2 + n^3)$.

\begin{algorithm}[htbp]
\caption{Implementation of MCE}
\label{alg:median_embedding}
\begin{algorithmic}[1]
\REQUIRE Embeddings $y_1, y_2, \dots, y_m$ and a sufficiently small constant $\varepsilon > 0$
\ENSURE Optimal distance matrix $\hat{x}$ and the corresponding embedding $\hat{y}$
\FOR{$i=1$ \TO $m$}
    \FOR{$k=1$ \TO $n$}
        \FOR{$l=1$ \TO $n$}
            \STATE $X(y_i)_{(k,l)} \gets \| y_{i (k)} - y_{i (l)} \|$
        \ENDFOR
    \ENDFOR
\ENDFOR
\STATE $x^{(0)} \gets \frac{1}{m}\sum_{i=1}^{m}X(y_i)$
\FOR{$t=1$ \TO \text{maximum number of iterations}}
    \STATE $\displaystyle w_i \gets {1} / {(\| x^{(t-1)} - X(y_i)\|_F + \varepsilon)}$
    \STATE $\displaystyle x^{(t)} \gets {\sum_{i=1}^{m}w_i X(y_i)} / {\sum_{i=1}^{m}w_i}$
    \IF{the convergence criterion is satisfied}
        \STATE \textbf{break}
    \ENDIF
\ENDFOR
\STATE $\hat{x} \gets x^{(t)}$
\STATE $\hat{y} \gets$ $\mathrm{MultiDimensionalScaling}(\hat{x})$
\RETURN $\hat{x}$, $\hat{y}$
\end{algorithmic}
\end{algorithm}

\section{Illustration on Actual Data}
\label{sec:illustration}

We apply the MCE method, as implemented in Algorithm~\ref{alg:median_embedding}, to real datasets.
We first demonstrate rapid convergence and reduction in embedding instability as the number of randomly initialized embeddings increases.
We next investigate a combined approach with multiple imputation to address missing values.
Finally, we apply the method to embeddings obtained under multiscale hyperparameters in dimensionality reduction algorithms.

\subsection{Data}

We used two publicly available biological datasets described below.

\paragraph{ToxoLopit.}

The ToxoLopit dataset is a whole-cell spatial proteomics dataset derived from the hyperplexed localization of organelle proteins by isotope tagging (hyperLOPIT) experiment on \textit{Toxoplasma gondii} extracellular tachyzoites~\citep{barylyuk2020comprehensive}. 
We extracted $718$ proteins assigned to distinct subcellular structures and characterized by $30$-dimensional tandem mass tag (TMT)-labeled peptide profiles.

\paragraph{Embryoid body.}

The Embryoid body dataset consists of single-cell RNA sequencing measurements for human embryonic stem cells differentiating as embryoid bodies over a $27$-day time course~\cite{moon2019visualizing}.
The dataset contains $16,825$ cells profiled at multiple time points, where each cell is represented by a $17,580$-dimensional vector of gene expression counts.
We subsampled the dataset to $10\%$ of the original size, resulting in $1,682$ cells for the experiments.

\subsection{Methods}

\paragraph{Evaluation of Rapid Convergence and Stability.}

We first applied $t$-SNE to ToxoLopit data and UMAP to Embryoid body data to embed these data in two dimensions $1000$ times with different random initializations.
We then performed MCE using these $1000$ embeddings as inputs, denoting the results as $\hat{y}_{1000}$.
We then obtained $10$ different embeddings using $t$-SNE (ToxoLopit data) or UMAP (Embryoid body data) and MCE with $2,\,10,\,20,\,50,\,100$ embeddings.
Suppose $\{\hat{y}^{(1)},\ldots,\hat{y}^{(10)}\}$ denotes a set of $10$ different embeddings, we calculated the following quantities:
\begin{itemize}
    \item[(1)] mean distance to $\hat{y}_{1000}$:
    $
        \sum_{i=1}^{10} d(\hat{y}^{(i)},\hat{y}_{1000}) / 10
        ,
    $
    \item[(2)] mean pairwise distance between embeddings:
    $
        \sum_{1\leq i < j \leq 10} d(\hat{y}^{(i)},\hat{y}_{(j)}) / \binom{10}{2}
        .
    $
\end{itemize}

$t$-SNE was performed using scikit-learn with the $\mathrm{perplexity}=30$~\citep{pedregosa2011scikitlearn}.
The hyperparameter setting of UMAP was $\mathrm{n\_neighbors}=15$ and $\mathrm{min\_dist}=0.1$, where $\mathrm{n\_neighbors}$ is the size of the local neighbor data points and $\mathrm{min\_dist}$ defines the minimum separation between points in the embedding.
The initial values were generated as the default settings of the scikit-learn and umap-learn libraries.
For $t$-SNE, the initial values were drawn from a $2$-dimensional normal distribution $\mathcal{MVN}(\mathbf{0},\,10^{-4}\times I)$, where $\mathbf{0}$ denotes the zero vector and $I$ is the $2$-dimensional identity matrix.
For UMAP, the initial values of each axis were independently generated from a uniform distribution $\mathcal{U}(-10, 10)$.
The MCE was implemented in Python 3.9.0 using Algorithm~\ref{alg:median_embedding}.

We further evaluated the instability arising from MDS.
For each combination of data and embedding method, the MCE distance matrix was obtained with $50$ embeddings with different initializations.
We then conducted MDS $100$ times using the SMACOF algorithm.
For each execution, the final embedding was selected as the best among $4$ trials initialized independently from a uniform distribution $\mathcal{U}(0, 1)$ for each axis.
We evaluated the mean pairwise distance between the $100$ embeddings from MDS.

\paragraph{Combined Approach with Multiple Imputation.}

To evaluate the practical utility of our proposed method in scenarios involving missing values, we investigated an approach combining multiple imputation (MI) with MCE, similar to the combination of MI with consensus clustering~\citep{audigier2022clustering}.
We applied the approach to the ToxoLopit dataset.
As the original ToxoLopit dataset is complete, we artificially introduced missing values to simulate realistic data analysis scenarios.
We investigated two missingness mechanisms:
\begin{itemize}
    \item[(1)] {random missingness}: values were deleted with a uniform probability, simulating a missing completely at random (MCAR) mechanism.
    \item[(2)] {intensity-dependent missingness}: values with lower intensity were more likely to be missing, simulating a missing not at random (MNAR) mechanism often observed in proteomics.
    The details of the generation of missing values are described below.
\end{itemize}
In the intensity-dependent missingness scenario, we defined a threshold $\tau \in \mathbb{R}$ corresponding to the $30$th percentile of the entire data distribution.
The probability $\pi_{ij} \in [0,1]$ that an entry $x_{ij} \in \mathbb{R}$ was missing was determined by a step function dependent on the target missing rate $\varrho \in [0,1]$:
\begin{align*}
    \pi_{ij}
    =
    \mathrm{Pr}(\phi_{ij}=1 \mid x_{ij})
    =
    \begin{cases}
    \min(1, 2\varrho), & \quad \text{if}\quad x_{ij} < \tau
    ,
    \\
    \varrho / 2, & \quad\text{if}\quad x_{ij} \geq \tau
    ,
    \end{cases}
\end{align*}
where $\phi_{ij} \in \{0,1\}$ denotes the missing indicator.
This formulation ensures that values with lower intensity are significantly more likely to be missing, thereby simulating an MNAR mechanism often observed in proteomics due to detection limits.
For each mechanism, we used missing rates of $\varrho \in \{0.1,0.3\}$.

For the imputation step, we employed the multivariate imputation by chained equations (MICE) algorithm implemented as the IterativeImputer function in scikit-learn.
We used the Bayesian ridge regression models for developing imputation models.
We sampled imputation values from posterior distributions truncated with the maximum and minimum of the observed values ($\mathrm{sample\_posterior}=\mathrm{True}$), and we generated $m=50$ imputed datasets for each missing data scenario.
Subsequently, we applied $t$-SNE to each of the $50$ imputed datasets with $\mathrm{perplexity} =30$.
The initial values for $t$-SNE were drawn from $\mathcal{MVN}(\mathbf{0},\,10^{-4}\times I)$ as in the previous experiments.
We then computed the MCE from these $50$ embeddings to obtain a single consensus embedding.

This procedure was repeated $20$ times for each missingness scenario (combination of mechanism and rate).
We further performed $1000$ independent $t$-SNE runs on the original complete ToxoLopit data and performed MCE to obtain the base embedding: $\hat{y}_{1000}$.
The performance was evaluated by calculating the mean distance to $\hat{y}_{1000}$: $\sum_{i=1}^{20} d(\hat{y}_{\mathrm{MI}}^{(i)},\hat{y}_{1000}) / 10$, where $\hat{y}_{\mathrm{MI}}^{(i)}$ denotes the consensus embedding obtained from $i$th repetition.

\begin{remark}
    Multiple imputation is a 2-step procedure: constructing a posterior predictive distribution for the missing values and imputing them via random sampling.
    Let $\mu_{\mathrm{imp}}$ denote the probability measure on the embedding space $\tilde{\mathcal{Y}}$ induced by the variability of these draws from the predictive distribution.
    Additionally, let $\mu_{\mathrm{init}}$ denote the probability measure associated with the random initialization of the $t$-SNE algorithm.
    Since the imputation and initialization processes are independent, the stochastic behavior of the resulting embeddings in this combined framework is governed by the product measure $\mu = \mu_{\mathrm{imp}} \times \mu_{\mathrm{init}}$.
    Therefore, Theorem~\ref{thm:main} can be applied to this setting.
\end{remark}

\paragraph{Consensus Representation of Multiscale Hyperparameter.}

We further investigated the capability of MCE to integrate embeddings generated with different hyperparameters into a single multiscale consensus representation.
The $\mathrm{perplexity}$ parameter in $t$-SNE is interpreted as it controls the effective number of nearest neighbors.
While configuring a small $\mathrm{perplexity}$ value preserves the local structure of the dataset, employing a larger value tends to reflect the global structure to the resulting embeddings.
As the $t$-SNE is sensitive to the perplexity value, determining the value is a practical issue.
Multiscale approaches may be a solution to this problem.

In this illustration, we applied $t$-SNE to the ToxoLopit data using a set of perplexity values $\{10, 30, 90, 270\}$.
For each perplexity value, we generated $20$ independent embeddings with different random initializations.
Consequently, $m = 80$ embeddings ($4$ perplexity settings $\times$ $20$ runs) were obtained, and we applied MCE.

\subsection{Results}

\paragraph{Evaluation of Rapid Convergence and Stability.}

Figure~\ref{fig:mce_1000} shows visualizations of $\hat{y}_{1000}$.
In these figures, the color of each point corresponds to its label: assigned subcellular structures for ToxoLopit data and time points for Embryoid body data.

For ToxoLopit data, the mean distances to $\hat{y}_{1000}$ were $4.76$ (SD: $1.28$; single $t$-SNE), $4.37$ (SD: $1.11$; MCE with $2$ $t$-SNE embeddings), $1.53$ (SD: $0.67$; MCE with $10$ $t$-SNE embeddings), $1.28$ (SD: $0.39$; MCE with $20$ $t$-SNE embeddings), $0.797$ (SD: $0.199$; MCE with $50$ $t$-SNE embeddings), and $0.570$ (SD: $0.166$; MCE with $100$ $t$-SNE embeddings).
The mean pairwise distances were $6.04$ (SD: $2.87$; single $t$-SNE), $6.18$ (SD: $1.66$; MCE with $2$ $t$-SNE embeddings), $2.23$ (SD: $0.95$; MCE with $10$ $t$-SNE), $1.84$ (SD: $0.65$; MCE with $20$ $t$-SNE embeddings), $1.06$ (SD: $0.32$; MCE with $50$ $t$-SNE embeddings), and $0.681$ (SD: $0.220$; MCE with $100$ $t$-SNE embeddings).
For Embryoid body data, the mean distances to $\hat{y}_{1000}$ were $3.06$ (SD: $1.45$; single UMAP), $2.41$ (SD: $1.87$; MCE with $2$ UMAP embeddings), $0.810$ (SD: $0.067$; MCE with $10$ UMAP embeddings), $0.567$ (SD: $0.051$; MCE with $20$ UMAP embeddings), $0.383$ (SD: $0.051$; MCE with $50$ UMAP embeddings), and $0.273$ (SD: $0.020$; MCE with $100$ UMAP embeddings).
The mean pairwise distances were $4.49$ (SD: $1.73$; single UMAP), $3.67$ (SD: $2.28$; MCE with $2$ UMAP embeddings), $1.16$ (SD: $0.10$; MCE with $10$ UMAP), $0.788$ (SD: $0.050$; MCE with $20$ UMAP embeddings), $0.526$ (SD: $0.058$; MCE with $50$ UMAP embeddings), and $0.359$ (SD: $0.025$; MCE with $100$ UMAP embeddings).
These values are plotted in Figure~\ref{fig:emb_inst}.

These results indicate a rapid convergence and reduction in embedding instability for the MCE with the number of integrated embeddings.
Both evaluation metrics decreased sharply up to $m=10$, after which the rate of decrease gradually declined.
Consequently, using $m=10$ or $m=20$ may be sufficient for achieving adequate integration stability.

In addition, we assessed the instability of the MDS in the final step of our algorithm.
For the ToxoLopit data, based on the fixed consensus distance matrix derived from $m=50$ $t$-SNE embeddings, the mean pairwise distance among $100$ independent MDS runs was $0.316$ (SD: $0.232$).
Similarly, for the Embryoid body data with UMAP, the corresponding mean pairwise distance was $1.27$ (SD: $0.284$).
These values represent the variability arising from the stochastic initialization of the SMACOF algorithm, but it is relatively small.

\begin{figure}[htbp]
\centering
    \begin{subfigure}[b]{0.50\linewidth}
      \centering
      \includegraphics[width=\linewidth]{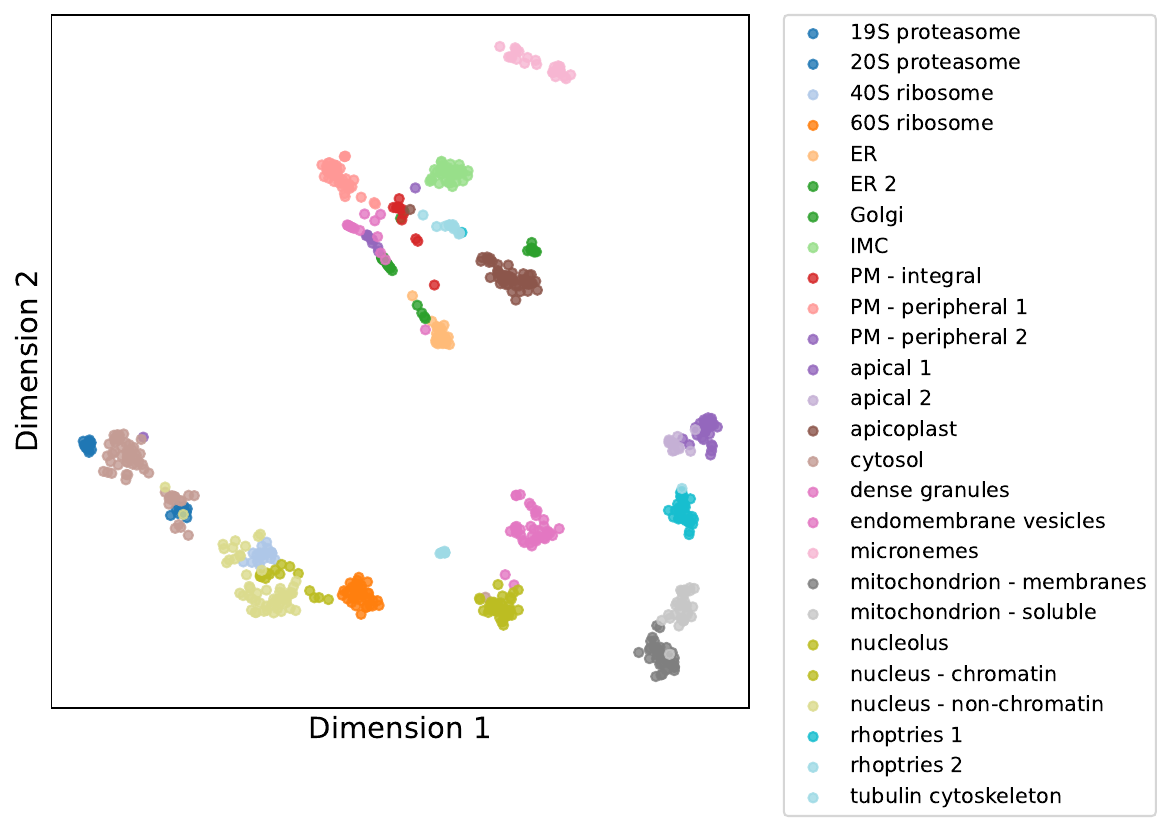}
      \caption{ToxoLopit.}
    \end{subfigure}
    \hspace{0.5em}
    \begin{subfigure}[b]{0.45\linewidth}
      \centering
      \includegraphics[width=\linewidth]{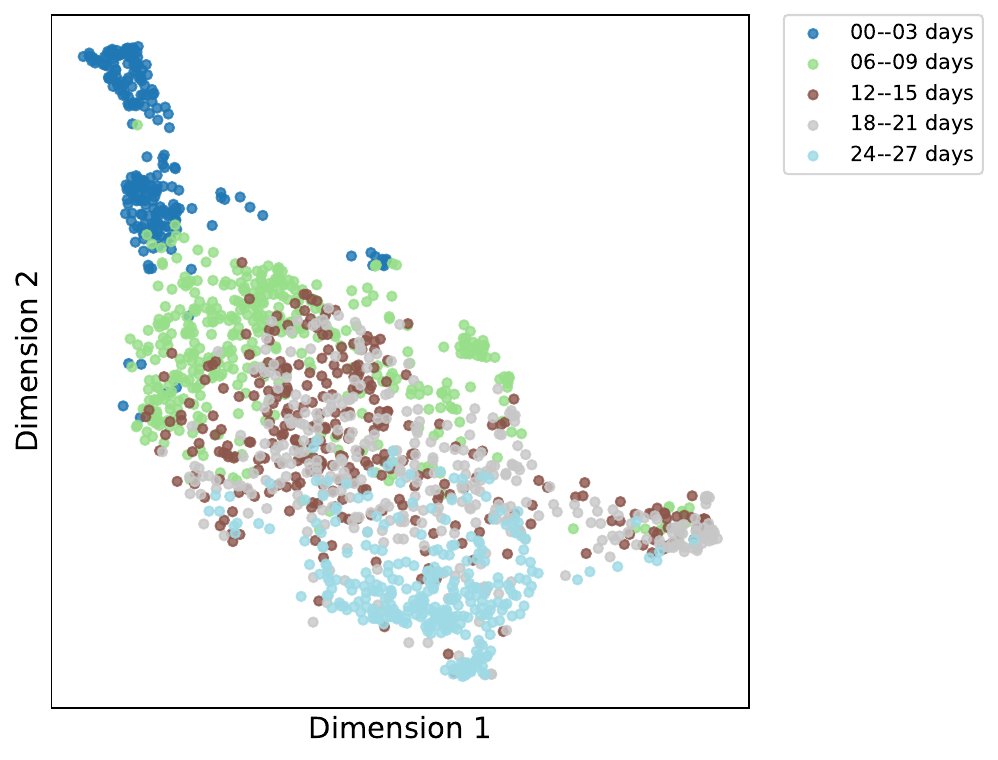}
      \caption{Embryoid body.}
    \end{subfigure}
\caption{Visualization of the embedding of the datasets obtained by the MCE with $1000$ embeddings. The left column (a) shows the results obtained using ToxoLopit data and $t$-SNE, and the right column (b) shows the results obtained using Embryoid body data and UMAP.}
\label{fig:mce_1000}
\end{figure}

\begin{figure}[htbp]
  \centering
  \begin{subfigure}[b]{1\linewidth}
      \centering
      \begin{subfigure}[b]{0.48\linewidth}
          \centering
          \includegraphics[width=\linewidth]{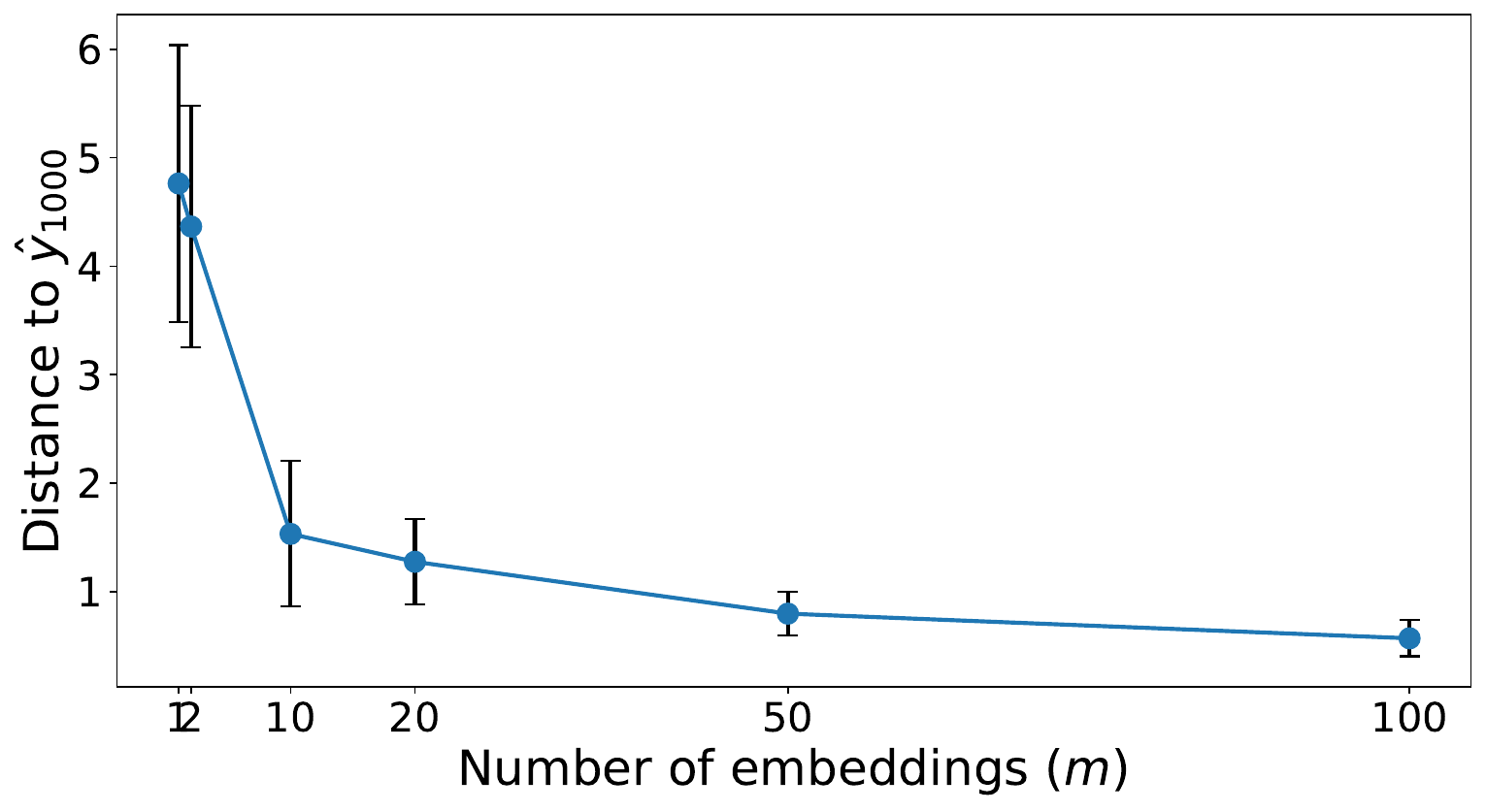}
      \end{subfigure}
      \begin{subfigure}[b]{0.48\linewidth}
          \centering
          \includegraphics[width=\linewidth]{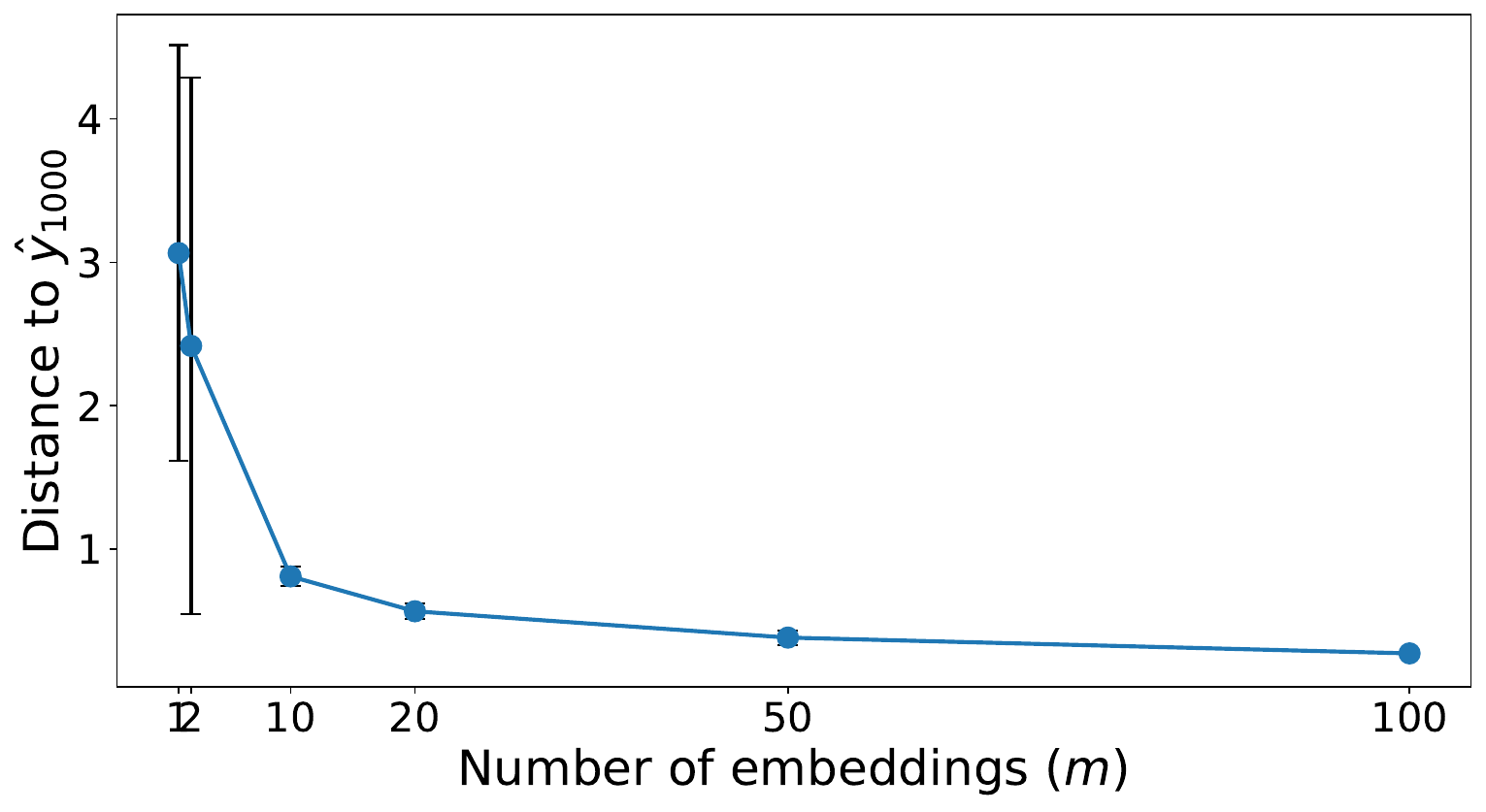}
      \end{subfigure}
  \end{subfigure}
  \\
  \vspace{0.5em}
  \begin{subfigure}[b]{1\linewidth}
      \centering
      \begin{subfigure}[b]{0.48\linewidth}
          \centering
          \includegraphics[width=\linewidth]{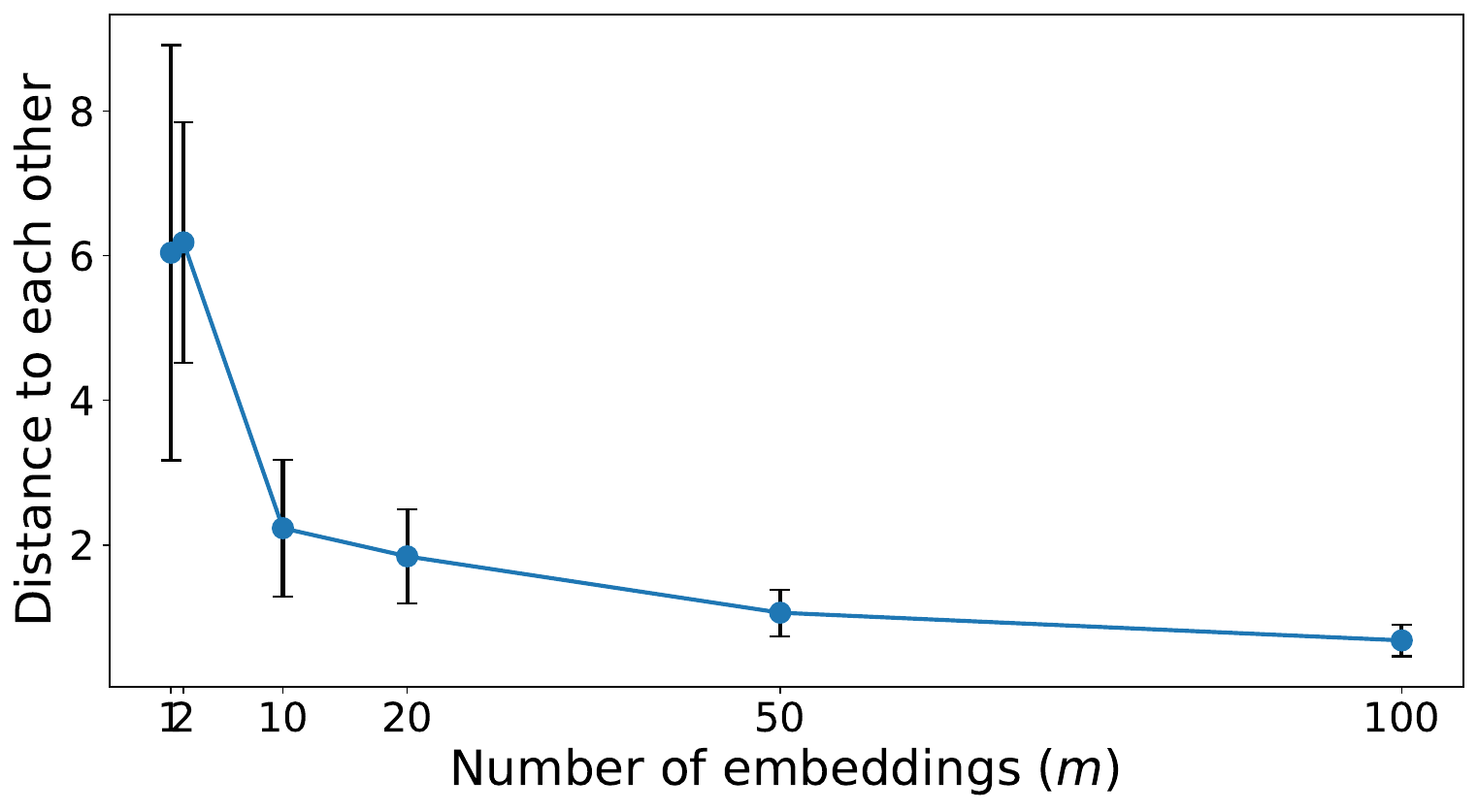}
      \subcaption{ToxoLopit.}
      \end{subfigure}
      \begin{subfigure}[b]{0.48\linewidth}
          \centering
          \includegraphics[width=\linewidth]{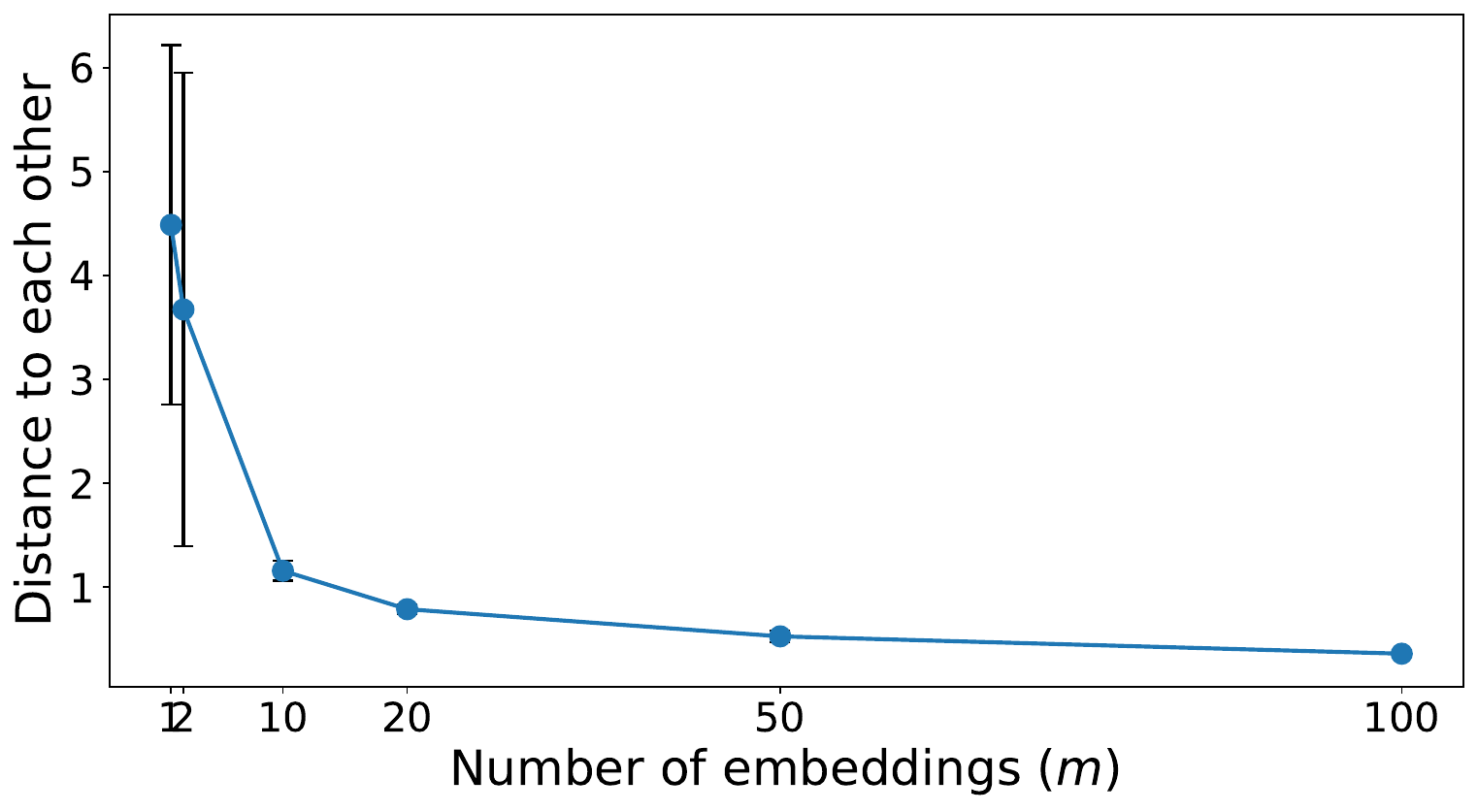}
      \subcaption{Embryoid body.}
      \end{subfigure}
  \end{subfigure}
 \caption{Mean distance to $\hat{y}_{1000}$ and mean pairwise distance among embeddings for $t$-SNE embeddings and MCE embeddings ($m=2,\,10,\,20,\,50,\,100$). The results are shown with error bars indicating standard deviations (SD). The left column (a) shows the results obtained using ToxoLopit data and $t$-SNE, and the right column (b) shows the results obtained using Embryoid body data and UMAP.}
 \label{fig:emb_inst}
\end{figure}

\paragraph{Combined Approach with Multiple Imputation.}

Figure~\ref{fig:mi} shows visualizations of the MCEs of the embeddings from multiple imputed datasets.
Visually, the results for the $10\%$ missing rate closely reproduced the original configuration observed in the complete data, preserving clear cluster separations.
For these scenarios, the mean distances to $\hat{y}_{1000}$ were $1.25$ (SD: $0.30$) for the random missingness case and $1.28$ (SD: $0.25$) for the intensity-dependent missingness case.
In the case of the $30\%$ missing rate, while the embeddings appeared slightly blurred, possibly due to the uncertainty associated with imputation, they still retained the overall topological structure and tendency of the original representation.
In these scenarios, the mean distances increased to $2.35$ (SD: $0.34$) for the random missingness case and $2.68$ (SD: $0.40$) for the intensity-dependent missingness case.
Despite the higher missing rate causing a relative deviation from the baseline, the proposed approach using MCE with multiple imputations consistently yielded stable embeddings that captured the underlying structures under both missingness mechanisms.

\begin{figure}[htbp]
\centering
    \begin{subfigure}[b]{0.30\linewidth}
      \centering
      \includegraphics[width=\linewidth]{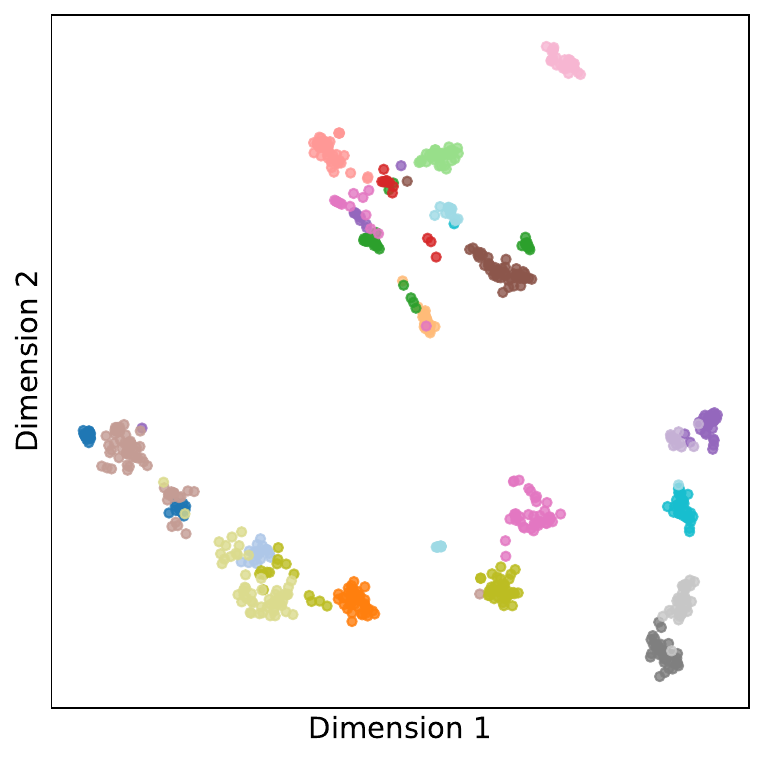}
      \caption{Random, 10\%.}
    \end{subfigure}
    \hspace{0.5em}
    \begin{subfigure}[b]{0.30\linewidth}
      \centering
      \includegraphics[width=\linewidth]{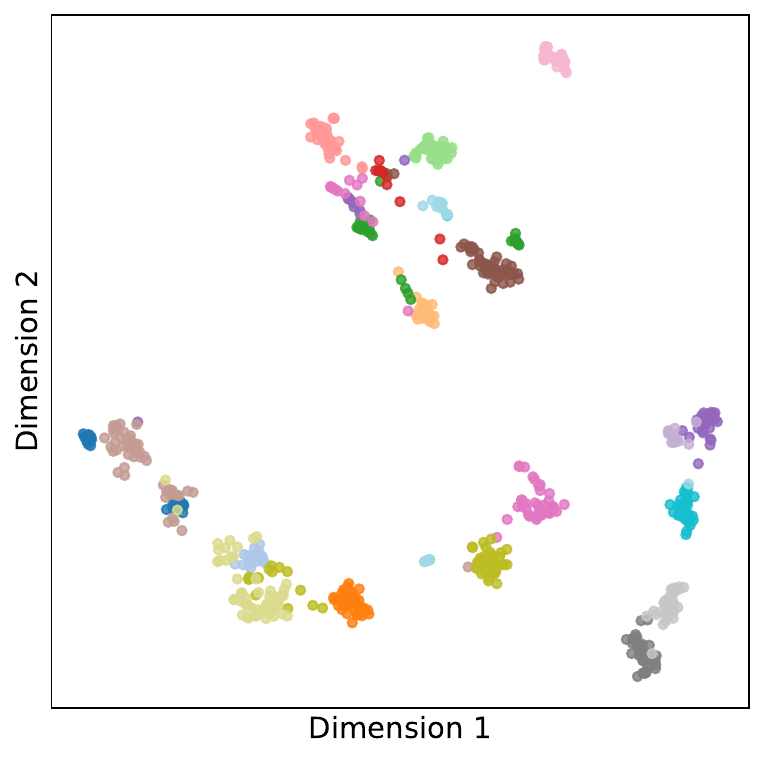}
      \caption{Intensity-dependent, 10\%.}
    \end{subfigure}
    \\
    \vspace{0.5em}
    \begin{subfigure}[b]{0.30\linewidth}
      \centering
      \includegraphics[width=\linewidth]{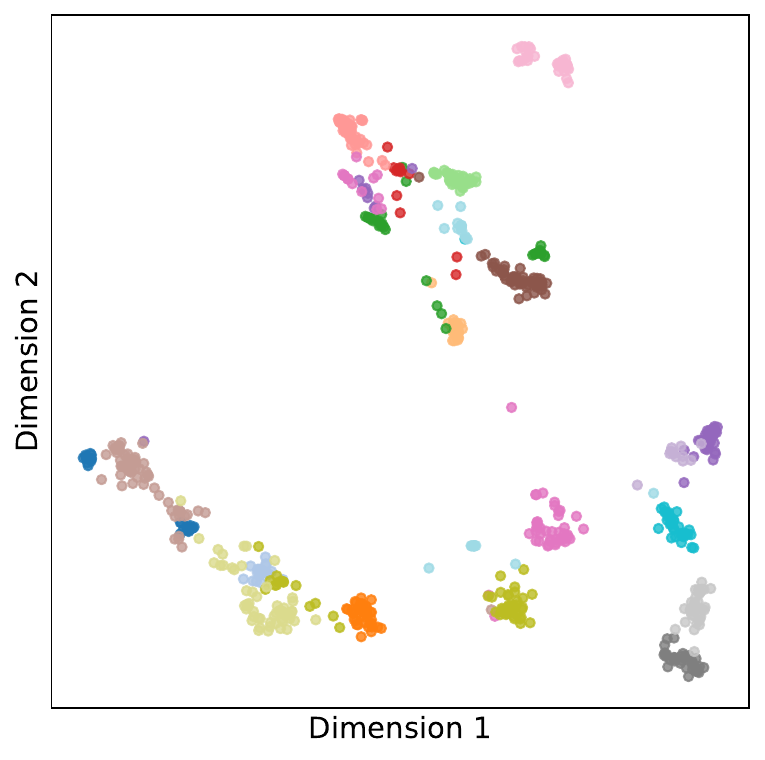}
      \caption{Random, 30\% .}
    \end{subfigure}
    \hspace{0.5em}
    \begin{subfigure}[b]{0.30\linewidth}
      \centering
      \includegraphics[width=\linewidth]{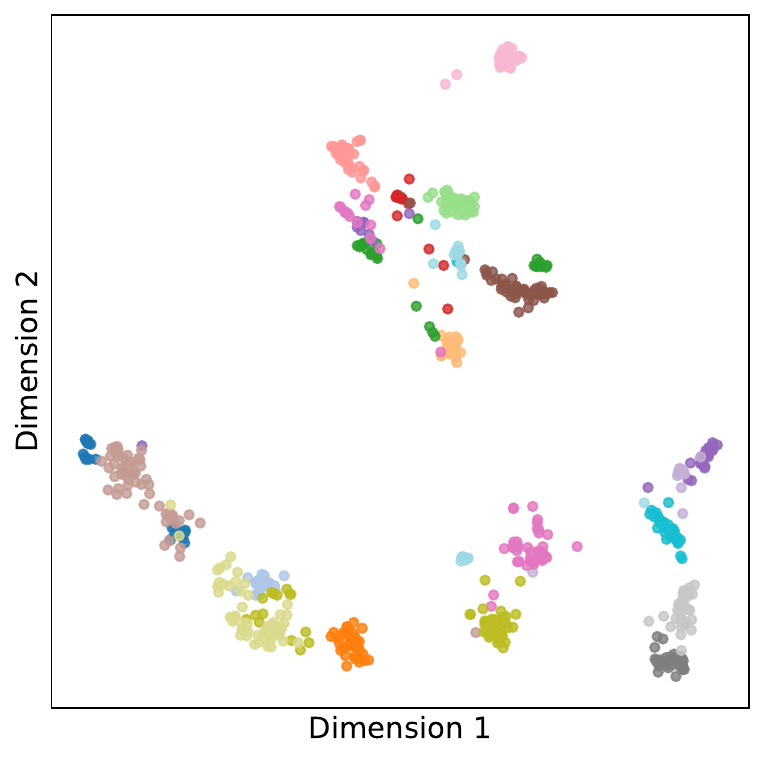}
      \caption{Intensity-dependent, 30\%.}
    \end{subfigure}
\caption{Visualization of the embeddings of the combined approach with multiple imputation applied to ToxoLopit data. (a) Random missing scenario with $10\%$ missing rate. (b) Intensity-dependent missing scenario with $10\%$ missing rate. (c) Random missing scenario with $30\%$ missing rate. (d) Intensity-dependent missing scenario with $30\%$ missing rate.}
\label{fig:mi}
\end{figure}

\paragraph{Consensus Representation of Multiscale Hyperparameter.}

Figure~\ref{fig:multiscale} presents the visual comparison between the MCE and representative embeddings for each perplexity setting based on ToxoLopit data.
Individual $t$-SNE results varied largely with the hyperparameter.
The embedding with $\mathrm{perplexity}=10$ exhibited fragmented local clusters with less coherent global organization, whereas the embedding with $\mathrm{perplexity}=270$ retained the global structure, resulting in a more continuous but less detailed cluster separation.
The MCE result (Figure~\ref{fig:mi}(a)) successfully synthesized these multiscale features.
In other words, the MCE maintained the global layout of high-perplexity embeddings while preserving the distinct local clustering structure observed at moderate perplexities.
The mean distances from the consensus embedding to the individual embeddings were $15.65$ (SD: $1.90$) for $\mathrm{perplexity}=10$, $7.90$ (SD: $0.31$) for $\mathrm{perplexity}=30$, $6.03$ (SD: $0.65$) for $\mathrm{perplexity}=90$, and $2.28$ (SD: $0.04$) for $\mathrm{perplexity}=270$.
The distance to the consensus decreased as the perplexity increased.
However, unlike the simple high-perplexity result, the MCE integrated information from lower scales to refine the boundaries of local clusters, achieving a balanced representation.

\begin{figure}[htbp]
\centering
    \begin{subfigure}[b]{0.30\linewidth}
      \centering
      \includegraphics[width=\linewidth]{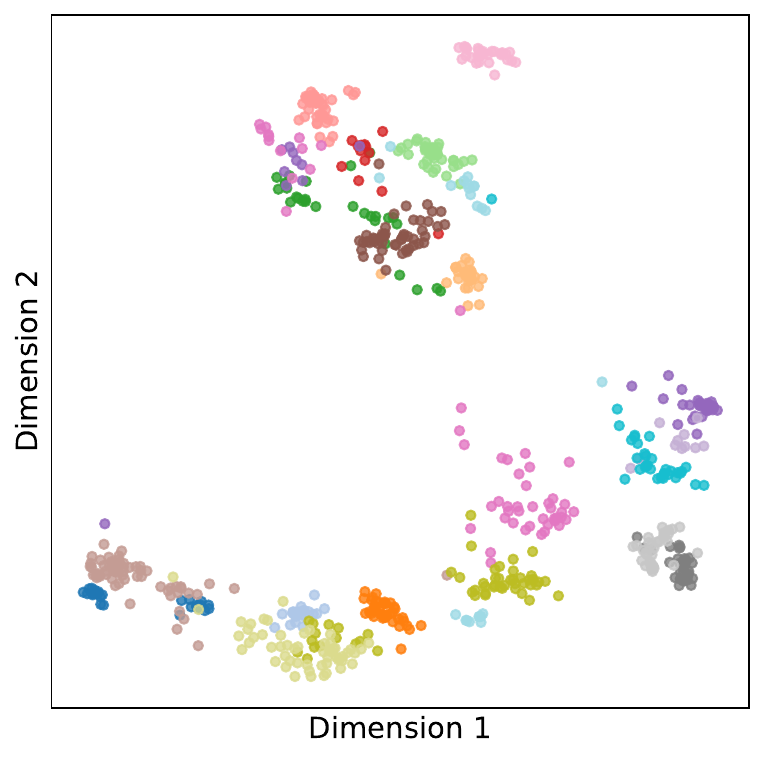}
      \caption{MCE.}
    \end{subfigure}
    \hspace{0.5em}
    \begin{subfigure}[b]{0.30\linewidth}
      \centering
      \includegraphics[width=\linewidth]{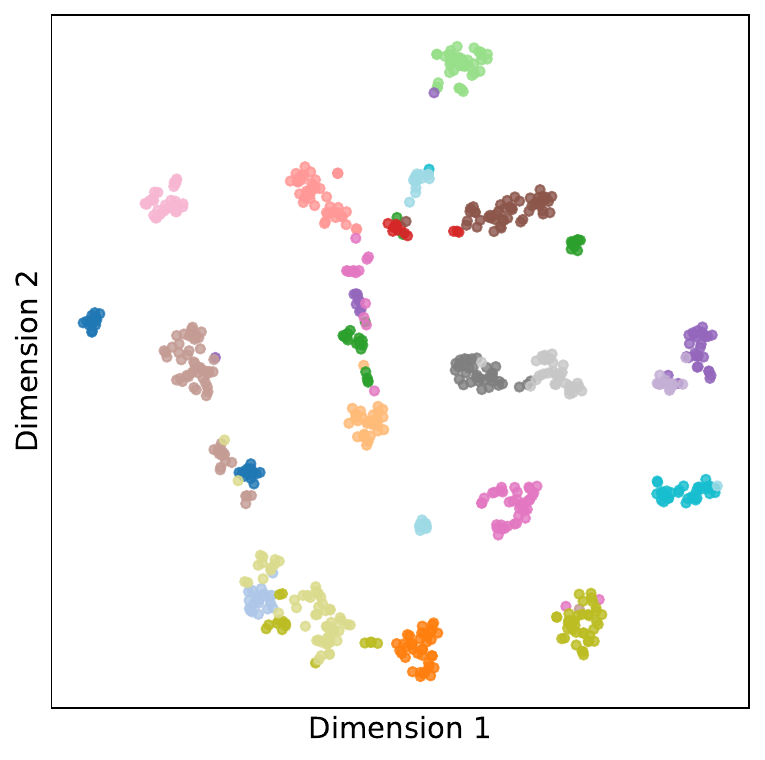}
      \caption{$\text{Perplexity}=10$.}
    \end{subfigure}
    \hspace{0.5em}
    \begin{subfigure}[b]{0.30\linewidth}
      \centering
      \includegraphics[width=\linewidth]{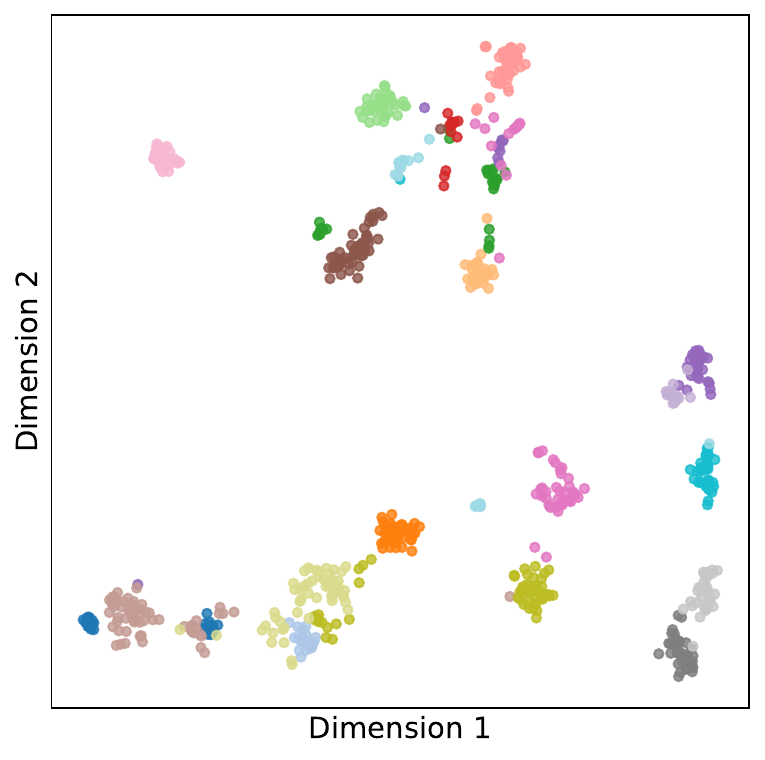}
      \caption{$\text{Perplexity}=30$.}
    \end{subfigure}
    \\
    \vspace{0.5em}
    \begin{subfigure}[b]{0.30\linewidth}
      \centering
      \includegraphics[width=\linewidth]{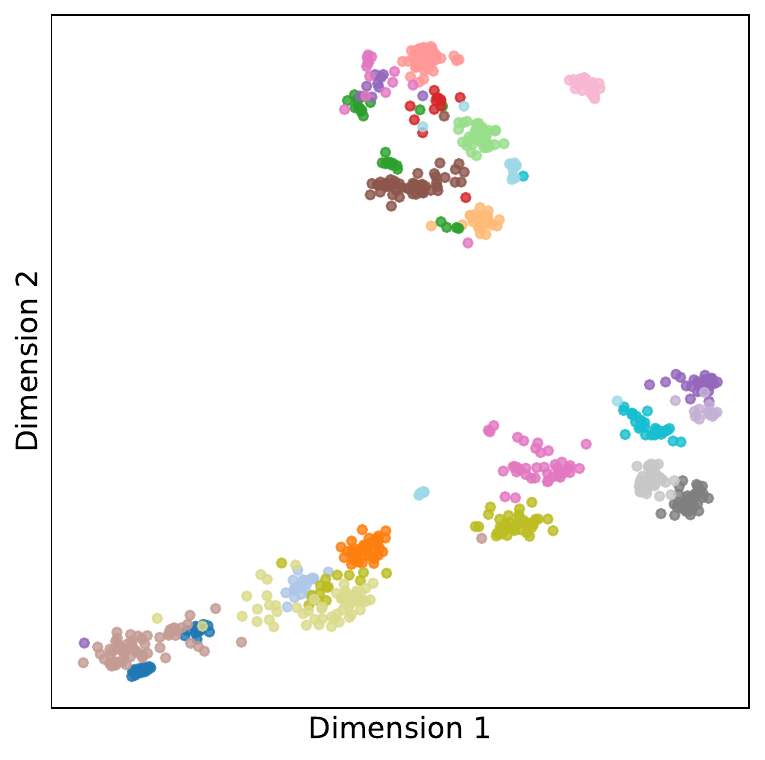}
      \caption{$\text{Perplexity}=90$.}
    \end{subfigure}
    \hspace{0.5em}
    \begin{subfigure}[b]{0.30\linewidth}
      \centering
      \includegraphics[width=\linewidth]{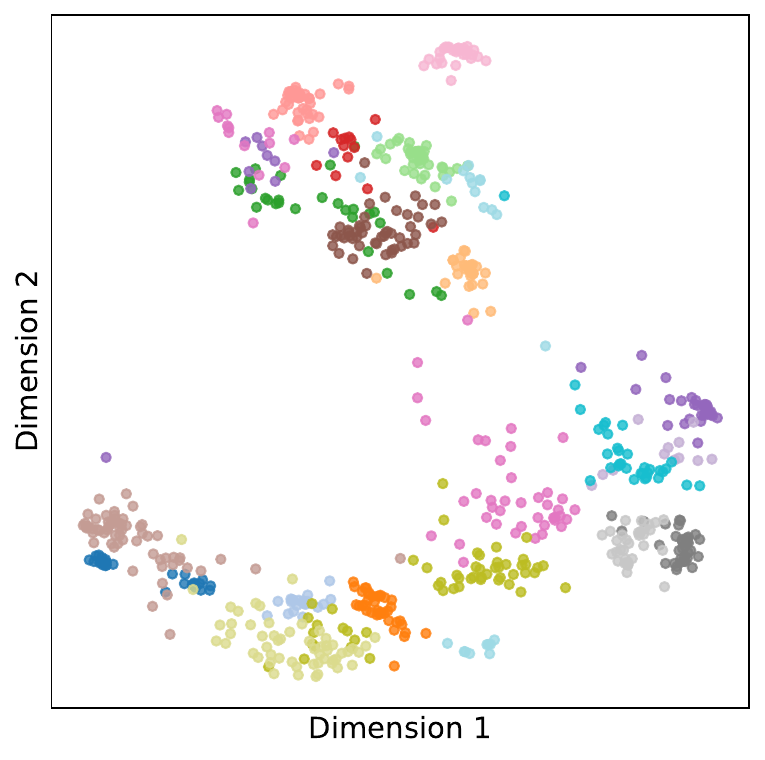}
      \caption{$\text{Perplexity}=270$.}
    \end{subfigure}
\caption{Visualization of the embeddings by $t$-SNE with multiple perplexity values and the MCE applied to ToxoLopit data. (a) MCE of the specified perplexity settings. (b) Embedding with $\text{Perplexity}=10$. (c) Embedding with $\text{Perplexity}=30$. (d) Embedding with $\text{Perplexity}=90$. (e) Embedding with $\text{Perplexity}=270$.}
\label{fig:multiscale}
\end{figure}

\section{Discussion}
\label{sec:discussion}

In this study, we proposed an MCE method for integrating multiple embeddings. 
By modeling them as i.i.d. samples from a probability measure and applying large deviation theory, our approach achieved consistency at an exponential rate.
Furthermore, we constructed a practical algorithm based on the Frobenius norm of the pairwise distance matrix of the data points.
By applying this method to real data, we demonstrated that embedding variability decreases as the number of embeddings $m$ increases.
We further demonstrated that the MCE can be combined with multiple imputations to address the problem of missing values and that the MCE produces a multiscale consensus embedding.
These results indicate that the consensus-embedding framework using a geometric median can effectively mitigate the instability of embeddings caused by random initialization or other sources.

Some approaches to generate multiscale embeddings have been developed.
For example, the multiscale $t$-SNE has been proposed to address this issue by averaging the similarity probability models of data points in high-dimensional space.
However, it does not address issues related to parameter settings in the optimization process~\citep{lee2015multi}.
Our approach addresses both these issues and is applicable to any dimensionality reduction method that has hyperparameters.

To apply our theoretical results, we must consider the following assumptions:
In Assumption~\ref{assump:trueembedding}, we define the distance space on $\tilde{\mathcal{Y}}$ rather than $\mathcal{Y}$ to ensure invariance under translation, scaling, and orthogonal transformations (rotations and reflections).
Therefore, if we assume the continuity of $\mu(y)$, it follows that almost surely,
\begin{align*}
    \int_{\tilde{\mathcal{Y}}} d(y',y) d\mu(y') \neq \int_{\tilde{\mathcal{Y}}} d(y',y^{*}) d\mu(y')
    ,
\end{align*}
when $y \neq y^{*}$.
Assumption~\ref{assump:mgf} may not hold if the underlying distribution of $Z_{y}(y_i)$ has a heavy tail.

The definition of true embedding used in this study differs from that in~\cite{viswanath2012consensus}.
In this study, we define "true" as the minimization of the expected distance regarding a probability measure.
Although this formulation is natural from a statistical perspective, it does not guarantee improvements in embedding evaluation metrics.
Some metrics for low-dimensional embeddings evaluate local neighborhood structure preservation, such as the average K-ary neighborhood preservation and co-k-nearest neighbor size, whereas others assess the cluster structure using label information, such as the silhouette coefficient~\citep{lee2009quality, zhang2021pydrmetrics, rousseeuw1987silhouettes}.
Future work will explore the practical properties of MCE across various datasets and evaluation metrics.

Although MCE can reduce the variability caused by random initialization, the results still depend on the probabilistic model used to generate the initial values.
Nonlinear dimensionality reduction methods typically generate low-dimensional embeddings that preserve certain structural properties of high-dimensional space.
For example, $t$-SNE and UMAP prioritize the local neighborhood structure preservation of each data point.
Consequently, global structures are typically lost in the resulting embeddings.
However, studies suggest that selecting appropriate initial values can enhance global structure preservation~\citep{kobak2019art}.
Investigating better initialization models remains a challenge for future research.

While we formulated our consensus approach using the geometric median of embeddings, other definitions of centrality are conceivable.
For instance, one could adopt the geometric mean (or Fr\'echet mean), defined as the minimizer of the mean squared loss rather than the sum of unsquared distances used in this study.
Although we selected the median owing to its robustness to outliers, a comparative analysis of these objective functions and their impact on the consensus embedding remains a topic for future research~\citep{minsker2015geometric}.

In conclusion, we proposed a general framework that yields a consensus embedding from multiple low-dimensional embeddings obtained from dimensionality reduction methods.
The proposed method is constructed based on a metric defined on an embedding space, providing a geometric formulation and a statistical guarantee on exponential-rate consistency.
Our method can address the instability of embeddings arising from random initialization, imputation of missing values, and the specification of hyperparameters and is expected to enhance the reliability of high-dimensional data analysis.

\section*{Acknowledgement}

We would like to thank Editage (www.editage.jp) for English language editing.

\section*{Disclosure statement}

The authors declare no conflicts of interest.

\section*{Code Availability}

The Python implementation of MCE is available from the GitHub repository \url{https://github.com/t-yui/MedianConsensusEmbedding}.

\section*{Data Availability}

The ToxoLopit dataset is available from the GitHub repository \url{https://github.com/lgatto/pRolocdata}.
The Embryoid body dataset is available from the GitHub repository \url{https://github.com/KrishnaswamyLab/PHATE}.

\renewcommand{\thesection}{\Alph{section}}
\setcounter{section}{0}

\section{Proof of Theorem~\ref{thm:main}}

We first establish the following lemma.

\begin{lemma}
\label{lem:technical_lemma_coverage}
For any $\varepsilon > 0$, there exist $N \in \mathbb{N}$ and the points 
$
    y^{(1)},\ldots,y^{(N)} \in \{ y \in \tilde{\mathcal{Y}} : d(y,y^*) \geq \epsilon \}
$
such that
\begin{align*}
    \mathrm{Pr}\left( d(\hat{y},y^*) \geq \epsilon \right)
    \leq \sum_{j=1}^N \mathrm{Pr}\left( \hat{S}^m_{y^{(j)}} \geq 0 \right)
    .
\end{align*}
\end{lemma}

\begin{proof}
As $\tilde{\mathcal{Y}}$ is compact,
$
    A := \{ y \in \tilde{\mathcal{Y}} : d(y,y^*) \geq \varepsilon \}
$
is also compact. 
According to the Heine--Cantor theorem, the function
\begin{align*}
    f(y) := \frac{1}{m}\sum_{i=1}^m d(y_i,y)
\end{align*}
is uniformly continuous in $\tilde{\mathcal{Y}}$.

We assume $d(\hat{y},y^*) \geq \varepsilon$.
Then, by the uniform continuity of $f$, for any $\eta > 0$, there exists $\delta > 0$ such that for all $y,y'\in\tilde{\mathcal{Y}}$,
\begin{align*}
    d(y,y') < \delta \quad \Longrightarrow \quad |f(y)-f(y')| < \eta
    .
\end{align*}
Furthermore, based on the definition of compactness, a finite collection of open balls exists: 
\begin{align*}
    B(y^{(j)},\delta)=\{ z \in \tilde{\mathcal{Y}} : d(z,y^{(j)}) < \delta\},
    \quad
    y^{(j)} \in A,\quad j=1,\dots,N_{\delta} \in \mathbb{N}
    ,
\end{align*}
such that
\begin{align*}
    A \subset \bigcup_{j=1}^{N_{\delta}} B(y^{(j)},\delta)
    .
\end{align*}
Because $\hat{y} \in A$, there exists some $j \in \{1,\ldots,N_{\delta}\}$ for which
\begin{align*}
    d(\hat{y},y^{(j)}) < \delta
    .
\end{align*}
Then, by using the uniform continuity of $f$, we have
\begin{align*}
    |f(\hat{y})-f(y^{(j)})| < \eta
    .
\end{align*}
Because $\hat{y}$ is a minimizer of $f$, it follows that
\begin{align*}
    f(\hat{y}) \leq f(y^{(j)})
    ,
\end{align*}
and hence,
\begin{align*}
    f(y^{(j)}) < f(\hat{y}) + \eta
    .
\end{align*}
Furthermore, as $f(\hat{y}) \leq f(y^*)$, we obtain
\begin{align*}
    f(y^{(j)}) < f(y^*) + \eta
    ,
\end{align*}
namely,
\begin{align*}
    \hat{S}^m_{y^{(j)}} = f(y^*) - f(y^{(j)}) > -\eta
    .
\end{align*}
Suppose that $\hat{S}^m_{y^{(j)}} < 0$. Then there exists $\gamma>0$ such that
\begin{align*}
    \hat{S}^m_{y^{(j)}} \leq -\gamma
    .
\end{align*}
As we can choose $\eta > 0$ arbitrarily, we choose $\eta = \gamma/2$, which results in
\begin{align*}
    -\gamma \geq \hat{S}^m_{y^{(j)}} > -\frac{\gamma}{2}
    ,
\end{align*}
which is contradictory.
Therefore, $\hat{S}^m_{y^{(j)}} \geq 0$.

From the discussion above, if $d(\hat{y},y^*) \geq \varepsilon$, then there exists $N \in \mathbb{N}$ and $y^{(j)} \in A$ such that $\hat{S}^m_{y^{(j)}} \geq 0$; that is,
\begin{align*}
    \{ d(\hat{y},y^*) \geq \varepsilon \}
    \subset
    \bigcup_{j=1}^N \{ \hat{S}^m_{y^{(j)}} \geq 0 \}
    .
\end{align*}
Taking the probabilities on both sides, we obtain
\begin{align*}
    \mathrm{Pr}\left( d(\hat{y},y^*) \geq \varepsilon \right)
    \leq
    \sum_{j=1}^N \mathrm{Pr}\left( \hat{S}^m_{y^{(j)}} \geq 0 \right)
    .
\end{align*}
\end{proof}

Next, we prove Theorem~\ref{thm:main} as follows. 

\begin{proof}[proof of Theorem~\ref{thm:main}]
For an i.i.d. random sequence $\{Z_{y}(y_i)\}_{i=1}^{m}$, where $Z_{y}(y_i) = d(y_i, y^{*}) - d(y_i,y)$, we define
\begin{align*}
    \hat{S}_{y}^{m} := \frac{1}{m} \sum_{i=1}^{m} Z_{y}(y_i).
\end{align*}
Based on Assumption~\ref{assump:mgf}, the logarithmic moment-generating function (cumulant-generating function) is given by
\begin{align*}
    \Lambda_{y}(\lambda)
    :=
    \log M_{y}(\lambda)
    :=
    \log \mathbb{E}_{\mu}[\exp(\lambda Z_y(y_i))]
    =
    \log \int_{\tilde{\mathcal{Y}}} \exp(\lambda Z_y(y_{i}')) d\mu({y_{i}'})
    .
\end{align*}
The Fenchel-Legendre transform of $\Lambda_{y}(\lambda)$ is defined as
\begin{align*}
    \Lambda_{y}^{*}(z) := \sup_{\lambda \in \mathbb{R}} \left\{\lambda z - \Lambda_{y}(\lambda)\right\}
    .
\end{align*}
The minimum value of $\Lambda_{y}^{*}(z)$ is $0$ and only occurs at
\begin{align}
    \label{eq:minimum_Lambda_star_z}
    z
    =
    \bar{z}_{y}
    :=
    \mathbb{E}_{\mu}[Z_{y}(y_i)] = \int_{\tilde{\mathcal{Y}}} \left\{d(y_i, y^{*}) - d(y_i,y)\right\} d\mu({y_{i}'})
    <
    0,
    \quad\text{from Assumption~\ref{assump:trueembedding}}
    .
\end{align}
From Cramer's theorem (Theorem 2.2.3 and Corollary 2.2.19 in~\cite{dembo2009large}),
for any $y \in \tilde{\mathcal{Y}}$, we obtain:
\begin{align*}
    \lim_{m \to \infty} \frac{1}{m} \log \mathrm{Pr}\left(\hat{S}_{y}^{m} \geq 0\right) = - \inf_{z \geq 0} \Lambda_{y}^{*}(z) 
    .
\end{align*}
Thus, for any $\tau_{y} >0$, there exists $M_{y} \in \mathbb{N}$ such that if $m \geq M_{y}$, then
\begin{align}
    \label{eq:cramer_limit_inequality}
    \mathrm{Pr}\left(\hat{S}_{y}^{m} \geq 0\right) \leq \exp \left(- m \left(\inf_{z \geq 0} \Lambda_{y}^{*}(z)  - \tau_{y}\right) \right)
    .
\end{align}
We define $\eta_{y}:= \inf_{z \geq 0} \Lambda_{y}^{*}(z) - \tau_{y}$.
Then, we have $\eta_{y} = \inf_{z \geq 0} \Lambda_{y}^{*}(z) - \tau_{y} > 0$ by choosing a sufficiently small $\tau_{y} > 0$, because $\Lambda_{y}^{*}(z) > 0$ for $z \geq 0 > \bar{z}_{y}$ from  (\ref{eq:minimum_Lambda_star_z}).

From Lemma~\ref{lem:technical_lemma_coverage}, for any $\varepsilon > 0$, there exist $N \in \mathbb{N}$ and points 
$
    y^{(1)},\ldots,y^{(N)} \in \{ y \in \tilde{\mathcal{Y}} : d(y,y^*) \geq \epsilon \}
    ,
$
and we obtain:
\begin{align*}
    \mathrm{Pr}\left( d(\hat{y},y^*) \geq \epsilon \right)
    \leq
    \sum_{j=1}^N \mathrm{Pr}\left( \hat{S}^m_{y^{(j)}} \geq 0 \right)
    .
\end{align*}
By applying (\ref{eq:cramer_limit_inequality}) to each $y^{(j)}$, we obtain:
\begin{align*}
    \mathrm{Pr}\left(d\left(\hat{y},y^*\right) \geq \epsilon\right)
    \leq
    \sum_{j = 1}^{N}\exp\left(-m\eta_{y^{(j)}}\right)
    \leq
    N \exp(-m \min_{j = 1,\ldots,N} \eta_{y^{(j)}})
    ,\quad \text{for } m > \max_{j = 1,\ldots,N} M_{y^{(j)}}
    .
\end{align*}
Therefore, by defining $M = \max_{j = 1,\ldots,N} M_{y^{(j)}} \in \mathbb{N}$, $K = N > 0$, and $\eta = \min_{j=1,\ldots,N} \eta_{y^{(j)}} > 0$, we conclude that if $ m > M$,
\begin{align*}
    \mathrm{Pr}\left(d(\hat{y}, y^*) \geq \epsilon\right)
    \leq
    K\exp\left(- m \eta \right)
    .
\end{align*}
\end{proof}

\section{Proof of Proposition~\ref{prop:distance_x}}

We establish the proof of Proposition~\ref{prop:distance_x} as follows:

\begin{proof}[Proof of Proposition~\ref{prop:distance_x}]
By definition,
\begin{align*}
    d(y_1,y_2)
    =
    D\left(\tilde{X}(y_1), \tilde{X}(y_2)\right)
    , 
    \quad y_1,y_2 \in \tilde{\mathcal{Y}}
    .
\end{align*}
As $D$ is a distance function on $\mathcal{X}$, $D$ satisfies the following properties:
\begin{itemize}
    \item[(I)] 
    for all $x_1, x_2 \in \tilde{\mathcal{X}}$,
    $
        D(x_1,x_2)=0 \, \Longleftrightarrow \, x_1 = x_2,
    $
    \item[(II)] 
    for all $x_1, x_2 \in \tilde{\mathcal{X}}$,
    $
        D(x_1,x_2) = D(x_2,x_1),
    $
    \item[(III)]
    for all $x_1, x_2, x_3 \in \tilde{\mathcal{X}}$,
    $
        D(x_1,x_2) + D(x_2,x_3) \geq D(x_1,x_3).
    $
\end{itemize}
Therefore, for any $y_1,y_2,y_3 \in \tilde{\mathcal{Y}}$,
\begin{align*}
    d(y_1,y_2)
    &= 
    D\left(\tilde{X}(y_1),\tilde{X}(y_2)\right)
    = 
    D\left(\tilde{X}(y_2),\tilde{X}(y_1)\right)
    =
    d(y_2,y_1)
    ,
    \\
    d(y_1,y_3)
    &=
    D\left(\tilde{X}(y_1),\tilde{X}(y_3)\right)
    \leq
    D\left(\tilde{X}(y_1),\tilde{X}(y_2)\right) + D\left(\tilde{X}(y_2),\tilde{X}(y_3)\right)
    \\
    &=
    d(y_1,y_2) + d(y_2,y_3)
    ,
\end{align*}
from (II) and (III), respectively.
This proves properties (ii) and (iii).

We now suppose $d(y_1,y_2) = 0$.
Then, 
\begin{align*}
    D\left(\tilde{X}(y_1),\tilde{X}(y_2)\right)
    =
    d(y_1,y_2)
    =
    0
    \quad \Longrightarrow \quad
    \tilde{X}(y_1)
    =
    \tilde{X}(y_2),
\end{align*}
from (I).
Let $y_1', y_2' \in \mathcal{Y}$ be representatives of $y_1 \in \tilde{\mathcal{Y}}$ and $y_2 \in \tilde{\mathcal{Y}}$, respectively.
Then, $X(y_1') = X(y_2')$ holds because $y_1'$ and $y_2'$ differ only by an orthogonal transformation.
Therefore, we have $y_1 \sim y_2$, so $y_1 = y_2$ in $\tilde{\mathcal{Y}}$.
Conversely, if $y_1 = y_2$, we immediately have
\begin{align*}
    d(y_1,y_2)
    =
    D\left(\tilde{X}(y_1),\tilde{X}(y_2)\right)
    =
    0
    .
\end{align*}
This proves (i).

Therefore, $d$ is a distance function on $\tilde{\mathcal{Y}}$.
\end{proof}

\section{Proof of Proposition~\ref{prop:rewrite_opt}}

We now prove Proposition~\ref{prop:rewrite_opt}.

\begin{proof}[Proof of Proposition~\ref{prop:rewrite_opt}]
By construction of the quotient space, if $z \in \pi(y)$ for some $y \in \mathcal{Y}$, then $z$ and $y$ differ only by an orthogonal transformation, and
\begin{align*}
    X(z)
    =
    X(y)
    .
\end{align*}
Thus, the mapping $\tilde{X} : \tilde{\mathcal{Y}} \to \mathcal{X}$ is well-defined.
Moreover, for any $x \in \mathcal{X}$, there exists $y \in \tilde{\mathcal{Y}}$ such that $x = \tilde{X}(y)$.
Therefore, $\tilde{X}$ is surjective.

For any $y \in \tilde{\mathcal{Y}}$, let $x := \tilde{X}(y)$, and for
$i=1,\ldots,m$ let $x_i := \tilde{X}(y_i)$, where $y_1,\ldots,y_m \in \tilde{\mathcal{Y}}$.
Using the definition of $d$, the objective function in \eqref{eq:consensus_embedding} can be rewritten as
\begin{align*}
    \frac{1}{m}\sum_{i=1}^{m} d(y_i,y)
    &=
    \frac{1}{m}\sum_{i=1}^{m}
    D\left(\tilde{X}(y_i),\tilde{X}(y)\right)
    =
    \frac{1}{m}\sum_{i=1}^{m} D(x_i,x).
\end{align*}
Hence, via the change of variable $x = \tilde{X}(y)$, optimization over $y \in \tilde{\mathcal{Y}}$ in \eqref{eq:consensus_embedding} is equivalent to optimization over $x \in \mathcal{X}$ in \eqref{eq:opt_X}.
In particular, if $\hat{y}$ is a minimizer of \eqref{eq:consensus_embedding} and $\hat{x}$ is a minimizer of \eqref{eq:opt_X}, then,
\begin{align*}
    \hat{x} = \tilde{X}(\hat{y}).
\end{align*}
\end{proof}

\bibliography{bibliography}
\bibliographystyle{apalike}

\end{document}